\documentclass[runningheads]{llncs}

\usepackage{multirow}
\usepackage[section]{placeins}
\usepackage{xcolor}
\newcommand{\gray}[1]{\textcolor{gray}{#1}}
\usepackage{accv}

\usepackage{accvabbrv}

\usepackage{graphicx}
\usepackage{booktabs}

\usepackage[accsupp]{axessibility}  

\usepackage[pagebackref,breaklinks,colorlinks,citecolor=accvblue]{hyperref}

\usepackage{orcidlink}

\begin{document}

\title{Erasing Without Collateral Damage: Precise Concept Removal in Diffusion Models} 

\titlerunning{Abbreviated paper title}

\author{Parth Upman\inst{1} \and
Nishita Jain \inst{2} \and
Shreyank N Gowda\inst{1}}

\authorrunning{Upman et al.}

\institute{School of Computer Science, University of Nottingham, Nottingham, NG8 1BB, UK \and
Department of Computing, Imperial College London, London, SW7 2AZ, UK 
\email{\{psxpu1,shreyank.narayanagowda\}@nottingham.ac.uk}}

\maketitle

\begin{abstract}
Training-free concept erasure is an attractive mechanism for controlling text-to-image diffusion models, but precise erasure often comes at the cost of damaging semantically related non-target concepts. Existing value-space methods remove the component of each cross-attention value along the target concept direction, implicitly treating target identity and shared visual structure as the same signal. We argue that this is the source of much of the collateral damage in prior preservation. We introduce \textbf{CARE}, a closed-form concept erasure operator that replaces the raw target direction with a kept-subspace-aware direction computed from a small bank of retained concept anchors. The resulting edit is applied directly in cross-attention value space, requires no model fine-tuning, and adds only a negligible offline computation. A single shrinkage parameter controls the erase–preserve trade-off. We further show that the operator admits a minimum-disturbance interpretation and, in its projection form, leaves the kept subspace invariant. Experiments under the standard concept-erasure protocol show that our method preserves non-target concepts more faithfully while maintaining competitive erasure across instance, style, and celebrity concepts. 
Code: \url{https://github.com/parthupman/care}

\end{abstract}

\section{Introduction}

Text-to-image diffusion models have become a practical interface for visual content creation, generating high-quality images from natural language prompts across objects, scenes, styles, identities, and abstract concepts~\cite{ho2020denoising,dhariwal2021diffusion,rombach2022highresolution,podell2024sdxl}. Their success relies partly on large-scale image-text pretraining and web-scale data collection~\cite{radford2021learning,schuhmann2022laion}, but this scale also creates a deployment problem: models trained on weakly curated corpora can reproduce copyrighted characters, distinctive artistic styles, public figures, unsafe content, or other concepts that a model owner may wish to restrict~\cite{birhane2021multimodal,carlini2023extracting,schramowski2023safe,mustafa2026low}. Since retraining a released model from scratch for every newly identified concept is rarely practical, \emph{concept erasure} has become an important problem for controllable generative modelling: remove a target concept while preserving the rest of the model's visual knowledge.

The central difficulty is that concept erasure is not simply about weakening generation. A useful method must be both effective and specific: prompts containing the target concept should no longer render it, while prompts for non-target concepts should remain faithful to the original model. These goals are naturally in tension. Removing a cartoon character should not damage other cartoon characters; erasing an artist style should not flatten related painterly styles; suppressing one public figure should not degrade face generation more broadly. The challenge is therefore to erase \emph{without collateral damage}.

Existing erasure methods approach this trade-off in different ways. Training-based methods fine-tune the diffusion model or attach lightweight modules to forget the target concept~\cite{gandikota2023erasing,kumari2023ablating,zhang2024forgetmenot,lyu2024one,lu2024mace}. They can achieve strong erasure, but usually require per-concept optimization and preserve the prior only through regularization, anchor losses, or adapters. Closed-form weight-editing methods avoid iterative fine-tuning by solving for changes to cross-attention projections~\cite{gandikota2024unified,gong2024reliable}, but the edit is still baked into model weights and typically depends on an explicitly specified preserve set. Training-free methods intervene during generation through classifier-free guidance, negative prompting, safety guidance, or prompt/embedding manipulation~\cite{ho2022classifierfree,schramowski2023safe,yoon2025safree}. These are attractive for deployment, but can be too coarse to separate a target concept from nearby non-target concepts that share visual structure.

A particularly promising direction is to edit internal representations rather than weights or prompts. Cross-attention layers mediate how text tokens influence image latents in latent diffusion models~\cite{rombach2022highresolution}. AdaVD shows that intervening in the value space of these layers provides a fast and training-free surface for concept erasure: it removes the component of each cross-attention value along the recorded target value direction, with an adaptive gate controlling when the edit is applied~\cite{wang2025precise}. However, the raw target direction need not represent only the concept to be removed. It can also contain structure shared with concepts that should be preserved: shape and texture statistics shared by characters, colour and brushstroke patterns shared by artists, or the generic face manifold shared by identities. Removing the full target direction therefore risks removing more than the target, which is precisely where prior preservation fails.

We address this issue with \textbf{CARE}, short for \textbf{Covariance-Aware Retained-subspace Erasure}. CARE keeps the same training-free value-space intervention surface as AdaVD, but changes the direction of erasure. Instead of erasing along the raw target value, CARE builds a small bank of retained concept anchors and uses their covariance structure to compute a kept-subspace-aware direction. This direction down-weights components of the target that lie in high-variance retained directions, while preserving components that distinguish the target from retained concepts. The resulting operator is closed-form, applied directly to cross-attention values, requires no model fine-tuning, and adds only a small offline computation.

The shrinkage parameter in CARE exposes a continuous erase--preserve trade-off. In one limit, the retained covariance becomes negligible and CARE recovers standard value-space erasure. In the other, the operator approaches a projection that removes the target component outside the retained subspace. This connects concept erasure to minimum-disturbance projection and covariance-weighted discrimination: the CARE direction can be viewed as a Fisher/Mahalanobis-style discriminative direction under a retained-concept covariance metric~\cite{fisher1936use,mahalanobis1936generalised,ledoit2004wellconditioned}. Computationally, it is obtained through a low-rank Woodbury solve, avoiding any large matrix inverse and preserving the training-free character of the method~\cite{woodbury1950inverting,boyd2004convex}.

In summary, CARE reframes concept erasure as choosing the right direction to forget. Rather than asking only how strongly to remove a target, we ask which part of the target direction should be removed at all. Our contributions are:
\begin{itemize}
    \item We identify shared target-retain structure as a source of collateral damage in value-space concept erasure, and introduce a covariance-aware retained-subspace direction for training-free concept removal.
    \item We derive CARE as a closed-form erase--preserve operator with a single shrinkage parameter, recovering standard value-space erasure as a limiting case and yielding a kept-subspace invariance property in its projection form.
    \item We evaluate CARE under the standard concept-erasure protocol across instance, style, and celebrity removal, showing improved non-target preservation with matched or competitive erasure.
\end{itemize}

\section{Related Work}

\textbf{Text-to-image diffusion and concept control.}
Modern text-to-image systems build on diffusion models~\cite{ho2020denoising,dhariwal2021diffusion,nichol2021improved,song2021scorebased} and latent diffusion architectures that make high-resolution generation practical~\cite{rombach2022highresolution,podell2024sdxl}. Large-scale image-text pretraining enables open-ended prompt following and broad visual coverage~\cite{radford2021learning,schuhmann2022laion,saharia2022photorealistic,ramesh2022hierarchical}, but also exposes models to copyrighted, offensive, private, and otherwise sensitive visual concepts~\cite{birhane2021multimodal,carlini2023extracting}. Prompt-level filters and safety guidance can help, but are brittle to prompt modifications, jailbreaks, or adversarially chosen inputs~\cite{schramowski2023safe,mustafa2026low,pham2024circumventing,tsai2024ring}. This motivates concept-erasure methods that intervene in the generative process itself, rather than relying only on input or output filtering.

\noindent\textbf{Concept erasure and model editing.}
A major line of work edits model parameters so that a target concept is forgotten. Erased Stable Diffusion fine-tunes the model with negative guidance~\cite{gandikota2023erasing}, Concept Ablation maps the target toward a generic anchor~\cite{kumari2023ablating}, and Forget-Me-Not suppresses target attention maps during fine-tuning~\cite{zhang2024forgetmenot}. More recent methods improve scalability and locality through lightweight or concept-specific modules, including SPM~\cite{lyu2024one}, MACE~\cite{lu2024mace}, and Receler~\cite{huang2024receler}; benchmarks such as UnlearnCanvas further emphasize evaluation of both forgetting and preservation~\cite{zhang2024unlearncanvas}. These methods can achieve strong erasure, but usually require per-concept optimization, additional parameters, or edited checkpoints, with preservation encouraged through regularization or anchor losses rather than enforced by the edit structure.

\noindent\textbf{Closed-form and training-free interventions.}
To reduce fine-tuning cost, closed-form methods edit cross-attention projections directly. Unified Concept Editing solves for weight changes that map targets away while preserving specified concepts~\cite{gandikota2024unified}, and RECE studies reliable and efficient concept erasure in text-to-image models~\cite{gong2024reliable}. These approaches avoid iterative fine-tuning, but still modify model weights and depend on an explicit preserve set or penalty. Training-free generation-time methods instead intervene during sampling, for example through classifier-free guidance, safety guidance, negative prompting, SuppressEOT, or adaptive guards~\cite{ho2022classifierfree,schramowski2023safe,li2024get,yoon2025safree}. They are attractive for deployment, but guidance- or embedding-space interventions can be too coarse to separate a target concept from nearby non-target concepts that share visual structure.

\noindent\textbf{Cross-attention, value space, and retained subspaces.}
Cross-attention controls how text conditions image latents in diffusion models, building on transformer attention~\cite{vaswani2017attention,rombach2022highresolution}. Image-editing methods have exploited this structure by manipulating attention maps or text-conditioning signals during generation~\cite{hertz2022prompt,mokady2023null,avrahami2022blended,chefer2023attend}. Closest to our work is AdaVD, which observes that cross-attention values act as a visual ``what'' pathway and erases a target by removing the component of each value vector along the recorded target direction~\cite{wang2025precise}. CARE keeps this fast, closed-form, training-free value-space intervention, but changes the direction of erasure: instead of using the raw target value, it computes a covariance-aware direction from retained anchors, down-weighting components important for concepts that should remain intact.

\noindent\textbf{Covariance-aware directions.}
Our formulation is related to classical discriminative projection and covariance whitening. Fisher's linear discriminant and Mahalanobis distance define directions using covariance information~\cite{fisher1936use,mahalanobis1936generalised}; shrinkage covariance improves conditioning in high dimensions~\cite{ledoit2004wellconditioned}; and the Woodbury identity enables efficient low-rank inverse computation~\cite{woodbury1950inverting}. Subspace-constrained editing has also been used to preserve behaviour during model editing, for example in recent language-model editing work~\cite{fang2025alphaedit}. We use these tools not as a classifier or weight-editing constraint, but as a closed-form erasure direction: retained anchors define the covariance structure to protect, and the target is removed relative to that structure.

\section{Method}

We study training-free concept erasure for a frozen text-to-image diffusion model. Given a target concept $c$, the goal is to prevent the model from rendering $c$ while preserving its ability to generate non-target concepts. We focus on inference-time edits inside the cross-attention layers of the denoising U-Net, where text tokens influence the visual content written into the latent image representation. Our method, \textbf{CARE} (\textbf{C}ovariance-\textbf{A}ware \textbf{R}etained-subspace \textbf{E}rasure), keeps the same intervention surface as value-space erasure, but replaces the raw target direction with a retained-subspace-aware direction computed from a small bank of kept concepts. Figure~\ref{fig:care_overview} gives an overview.

\begin{figure}[!t]
    \centering
    \includegraphics[width=\textwidth]{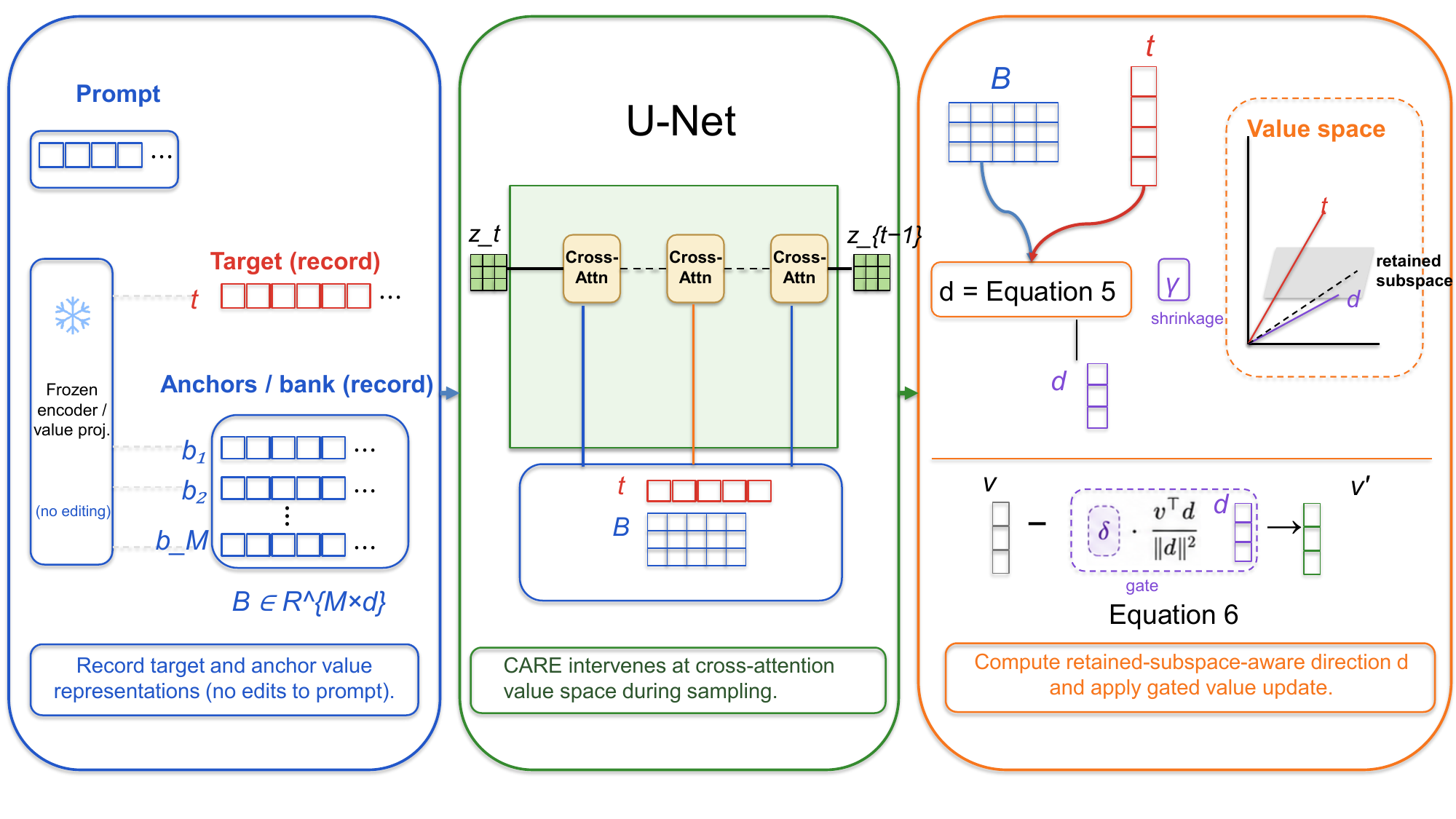}
    \caption{
    Overview of CARE. Given a target concept and a small bank of retained anchors, CARE records their value-space representations from the frozen text encoder and value projections, without editing the prompt or model weights. During sampling, CARE intervenes inside the cross-attention value space of the diffusion U-Net. The retained-anchor bank defines a covariance-aware direction 
    $d = \left(\frac{1}{M}B^\top B + \gamma I\right)^{-1}t$, 
    which down-weights components shared with retained concepts. The final value update applies a gated subtraction along this direction, suppressing the target-specific component while reducing collateral changes to retained concepts.
    }
    \label{fig:care_overview}
\end{figure}

\subsection{Value-Space Erasure Preliminaries}

Latent diffusion models inject text conditioning through cross-attention layers~\cite{rombach2022highresolution}. For a given cross-attention layer, let $v_j \in \mathbb{R}^{D}$ denote the value vector corresponding to token position $j$ in the prompt. Existing value-space erasure records a target value $t_j \in \mathbb{R}^{D}$ for concept $c$ and removes the component of $v_j$ along $t_j$:
\begin{equation}
    v^{\mathrm{erase}}_j
    =
    v_j
    -
    \delta\!\left(\cos(t_j, v_j)\right)
    \frac{\langle t_j, v_j\rangle}{\langle t_j, t_j\rangle}
    t_j ,
    \label{eq:raw_value_erasure}
\end{equation}
where $\delta(\cdot)$ is a token-wise gate that activates erasure when the prompt value is sufficiently aligned with the target direction:
\begin{equation}
    \delta(x) =
    \frac{s}{1 + \exp\{-p(x-\epsilon)\}} .
    \label{eq:gate}
\end{equation}
This edit is fast and training-free, but it treats $t_j$ as if it contains only the target concept. In practice, $t_j$ can also contain structure shared with concepts that should be preserved, such as cartoon structure, artistic texture, or the generic face manifold. CARE addresses this by changing the erasure direction.

\subsection{Retained Anchor Bank}

For each target concept, CARE constructs a small bank of retained concept anchors. These anchors represent concepts that should remain stable after erasure. For a given cross-attention layer and token position $j$, we record the value vectors of $M$ retained anchors and collect them as
\begin{equation}
    B_j =
    \begin{bmatrix}
        b_{j,1}^{\top} \\
        b_{j,2}^{\top} \\
        \vdots \\
        b_{j,M}^{\top}
    \end{bmatrix}
    \in \mathbb{R}^{M \times D},
    \label{eq:anchor_bank}
\end{equation}
where $b_{j,m}$ is the value vector of the $m$-th retained anchor at token position $j$. The target value $t_j$ and retained-anchor values are recorded offline using the frozen text encoder and value projections. No model weights are updated.

\subsection{Covariance-Aware Erasure Direction}

The anchor bank describes directions in value space that are important for retained concepts. CARE uses this structure to compute an erasure direction that remains aligned with the target while down-weighting components shared with retained concepts. We define the retained-anchor covariance as
\begin{equation}
    \Sigma_{R,j}
    =
    \frac{1}{M} B_j^{\top} B_j + \gamma I_D ,
    \label{eq:retained_covariance}
\end{equation}
where $\gamma > 0$ is a shrinkage parameter. CARE replaces the raw target direction $t_j$ with
\begin{equation}
    d_j = \Sigma_{R,j}^{-1} t_j .
    \label{eq:care_direction}
\end{equation}
This direction suppresses components of $t_j$ that lie in high-variance retained-anchor directions, while preserving components that distinguish the target from the retained anchors. Direct inversion of $\Sigma_{R,j}$ is unnecessary: since $M \ll D$, $d_j$ is computed efficiently with a low-rank Woodbury solve over an $M \times M$ matrix. The full derivation is provided in the supplementary material.

\subsection{CARE Value Update}

Given the covariance-aware direction $d_j$, CARE applies the same gated rank-one removal form as value-space erasure:
\begin{equation}
    v^{\mathrm{CARE}}_j
    =
    v_j
    -
    \delta\!\left(\cos(t_j, v_j)\right)
    \frac{\langle d_j, v_j\rangle}{\langle d_j, d_j\rangle}
    d_j .
    \label{eq:care_update}
\end{equation}
The gate remains based on the raw target similarity $\cos(t_j,v_j)$, preserving compatibility with prior value-space erasure. The selectivity of CARE comes from the erasure coefficient $\langle d_j, v_j\rangle$: if a prompt value does not contain the target-specific component, its projection onto $d_j$ is small, and the update has little effect. The update is applied independently at each edited cross-attention layer and token position, except for the first token position, which is skipped following prior value-space erasure.

\subsection{Erase--Preserve Trade-off}

The shrinkage parameter $\gamma$ controls the erase--preserve trade-off. When $\gamma \rightarrow \infty$, the covariance term becomes negligible and $\Sigma_{R,j}^{-1}t_j \propto t_j$, so CARE recovers standard raw target-direction erasure as a limiting case. As $\gamma$ decreases, retained-anchor covariance has a stronger effect and directions important for retained concepts are increasingly down-weighted. Thus large $\gamma$ favours erasure strength, while smaller $\gamma$ favours preservation of retained structure.

CARE is training-free. It does not update the U-Net, the text encoder, or any cross-attention projection weights. Its only additional cost over raw value-space erasure is the offline construction of $d_j$, after which inference uses the same type of gated rank-one subtraction. Additional details, including the Woodbury computation, the projection-form invariance property, and the multi-concept extension, are provided in the supplementary material.

\section{Experiments}

We evaluate CARE on training-free concept erasure for text-to-image diffusion models. The goal is to remove a target concept while preserving the model's ability to generate non-target concepts. We follow the evaluation protocol used by prior concept-erasure methods, reporting CLIP score (CS) for erased concepts and FID for retained concepts. Lower CS indicates stronger erasure of the target concept, while lower FID indicates better preservation of non-target concepts.

\subsection{Experimental Setup}

\textbf{Model and sampling.}
All experiments are conducted on Stable Diffusion v1.4 using the DPM-Solver sampler with 30 denoising steps and classifier-free guidance scale 7.5, following the standard AdaVD protocol~\cite{wang2025precise}. CARE is applied at inference time only. It does not update the text encoder, the UNet, or any cross-attention projection weights.

\noindent\textbf{Evaluation protocol.}
We evaluate three categories of concept erasure: instances, art styles, and celebrities. For instance erasure, we evaluate single- and multi-concept removal using Snoopy, Mickey, and Spongebob as erased concepts, with Pikachu, Dog, and Legislator used as non-target probes. For art-style erasure, we erase Van Gogh, Picasso, and Monet, and evaluate preservation on the remaining styles and held-out artists. For celebrity erasure, we erase Bruce Lee, Marilyn Monroe, and Melania Trump, and evaluate preservation on non-target identities. Unless otherwise stated, all main results use the same number of templates and samples as the standard protocol. The ablation studies use reduced-scope runs with 15 templates, NS=10, seed 0.

\noindent\textbf{Baselines and metrics.}
We compare against training-based erasure methods, including Concept Ablation~\cite{kumari2023ablating}, MACE~\cite{lu2024mace}, and SPM~\cite{lyu2024one}, and training-free methods, including negative prompting, SLD~\cite{schramowski2023safe}, and AdaVD~\cite{wang2025precise}. Baseline rows are transcribed from the AdaVD evaluation protocol, while CARE rows are produced under the same evaluation surface. In each table, target concepts are evaluated by CS$\downarrow$, and non-target concepts are evaluated by FID$\downarrow$.

\subsection{Instance Concept Erasure}

\begin{table}[!t]
\centering
\caption{Quantitative comparison of single- and multi-instance concept erasure. Each target concept is evaluated by CLIP score (CS, lower is better), while non-target concepts are evaluated by FID (lower is better). Best results are shown in bold and second-best results are underlined.}
\label{tab:instance_erasure}
\begin{tabular}{@{}l|cccccc@{}}
\toprule
\textbf{Concept} 
& \textbf{Snoopy} 
& \textbf{Mickey} 
& \textbf{Spongebob} 
& \textbf{Pikachu} 
& \textbf{Dog} 
& \textbf{Legislator} \\
\midrule
SD v1.4 
& 28.51 & 26.57 & 27.43 & -- & -- & -- \\
\midrule

\multicolumn{7}{@{}l}{\emph{Erase Snoopy}} \\
\cmidrule(lr){1-7}
& CS$\downarrow$ & FID$\downarrow$ & FID$\downarrow$ & FID$\downarrow$ & FID$\downarrow$ & FID$\downarrow$ \\
ConAbl~\cite{kumari2023ablating} 
& 25.38 & 38.44 & 41.59 & 29.68 & 27.76 & 27.36 \\
MACE~\cite{lu2024mace} 
& 20.78 & 118.01 & 111.90 & 81.99 & 43.27 & 65.97 \\
SPM~\cite{lyu2024one} 
& 23.89 & 33.06 & 34.70 & 23.89 & 19.61 & 18.26 \\
NP~\cite{schramowski2023safe} 
& 23.66 & 59.58 & 78.74 & 52.37 & 67.51 & 55.22 \\
SLD~\cite{schramowski2023safe} 
& 27.84 & 48.12 & 55.36 & 38.74 & 41.95 & 49.08 \\
AdaVD~\cite{wang2025precise} 
& \underline{20.28} & \underline{5.72} & \underline{8.56} & \underline{5.79} & \underline{2.32} & \underline{6.07} \\
\textbf{Ours $(\gamma = 0.5)$ }
& \textbf{19.42} & \textbf{4.43} & \textbf{7.42} & \textbf{4.74} & \textbf{1.95} & \textbf{5.32} \\
\midrule

\multicolumn{7}{@{}l}{\emph{Erase Snoopy and Mickey}} \\
\cmidrule(lr){1-7}
& CS$\downarrow$ & CS$\downarrow$ & FID$\downarrow$ & FID$\downarrow$ & FID$\downarrow$ & FID$\downarrow$ \\
ConAbl~\cite{kumari2023ablating} 
& 24.26 & 24.08 & 46.32 & 39.63 & 30.57 & 27.49 \\
MACE~\cite{lu2024mace} 
& 20.74 & \underline{20.71} & 51.49 & 110.67 & 52.07 & 77.13 \\
SPM~\cite{lyu2024one} 
& 23.16 & 22.81 & 41.58 & 31.77 & 21.96 & 23.69 \\
NP~\cite{schramowski2023safe} 
& 23.59 & 24.85 & 81.41 & 50.10 & 65.93 & 58.88 \\
SLD~\cite{schramowski2023safe} 
& 27.76 & 26.74 & 54.59 & 39.24 & 41.62 & 50.13 \\
AdaVD~\cite{wang2025precise} 
& \underline{20.29} & \textbf{19.93} & \underline{9.34} & \underline{5.84} & \underline{2.41} & \underline{6.43} \\
\textbf{Ours $(\gamma = 0.5)$ } 
& \textbf{19.45} & 22.55 & \textbf{7.47} & \textbf{4.98} & \textbf{2.18} & \textbf{5.34} \\
\midrule

\multicolumn{7}{@{}l}{\emph{Erase Snoopy and Mickey and Spongebob}} \\
\cmidrule(lr){1-7}
& CS$\downarrow$ & CS$\downarrow$ & CS$\downarrow$ & FID$\downarrow$ & FID$\downarrow$ & FID$\downarrow$ \\
ConAbl~\cite{kumari2023ablating} 
& 23.94 & 23.64 & 25.04 & 51.20 & 31.59 & 30.03 \\
MACE~\cite{lu2024mace} 
& 20.48 & \underline{20.50} & 21.59 & 99.68 & 47.46 & 70.38 \\
SPM~\cite{lyu2024one} 
& 22.81 & 22.35 & 20.82 & 39.83 & 22.68 & 25.31 \\
NP~\cite{schramowski2023safe} 
& 24.29 & 24.76 & 25.31 & 64.75 & 65.10 & 59.33 \\
SLD~\cite{schramowski2023safe} 
& 27.84 & 26.71 & 27.60 & \underline{39.41} & 42.32 & 49.88 \\
AdaVD~\cite{wang2025precise} 
& \textbf{19.39} & \textbf{19.73} & \underline{20.34} & \textbf{6.86} & \underline{2.79} & \underline{7.26} \\

\textbf{Ours $(\gamma = 0.5)$ }
& \underline{19.45} & 22.56 & \textbf{18.40} & \textbf{6.86} & \textbf{2.27} & \textbf{5.55} \\
\bottomrule
\end{tabular}%

\end{table}

Table~\ref{tab:instance_erasure} reports single- and multi-instance erasure. In the single-concept setting, CARE gives a clear improvement over AdaVD. When erasing Snoopy, CARE reduces the target CS from 20.28 to 19.42. More importantly, it improves preservation for every retained concept: Mickey FID decreases from 5.72 to 4.43, Spongebob from 8.56 to 7.42, Pikachu from 5.79 to 4.74, Dog from 2.32 to 1.95, and Legislator from 6.07 to 5.32. CARE strengthens erasure while reducing collateral changes across all retained probes. The multi-concept setting is more nuanced. When erasing Snoopy and Mickey jointly, CARE improves Snoopy erasure from 20.29 to 19.45, but Mickey erasure weakens from 19.93 to 22.55. At the same time, preservation improves consistently: Spongebob FID decreases from 9.34 to 7.47, Pikachu from 5.84 to 4.98, Dog from 2.41 to 2.18, and Legislator from 6.43 to 5.34. When erasing Snoopy, Mickey, and Spongebob jointly, CARE improves Spongebob erasure substantially, reducing CS from 20.34 to 18.40, and improves retained Dog and Legislator FID from 2.79 to 2.27 and 7.26 to 5.55 respectively. However, Mickey remains harder to erase in the multi-target setting.

These results show the intended behaviour of CARE. It is not designed to maximize erasure strength at any cost. Instead, it shifts the operating point toward better preservation, often with matched or improved erasure. This is especially important in concept-erasure settings where the target concept is visually close to concepts that should remain intact.


\subsection{Art-Style Erasure}

Table~\ref{tab:style_erasure} reports art-style erasure. This is the strongest setting for CARE. When erasing Van Gogh, CARE achieves the best target CS, reducing it from AdaVD's 24.87 to 23.63. It also improves preservation for Picasso, Monet, and Caravaggio, reducing FID from 6.82 to 5.86, 2.66 to 2.44, and 6.84 to 4.80 respectively. The only exception is Andy Warhol, where CARE obtains 8.59 compared with AdaVD's 8.36, a small difference.

For Picasso erasure, CARE matches AdaVD's target CS at 26.99, while improving all retained-style FIDs. Van Gogh FID decreases from 5.49 to 5.18, Monet from 2.33 to 1.69, Andy Warhol from 9.38 to 7.53, and Caravaggio from 7.05 to 4.36. Although negative prompting obtains a lower Picasso CS, it severely damages non-target styles, with retained FIDs above 90 in several columns. CARE therefore gives a much better erase--preserve operating point.

\begin{table}[!t]
\centering
\caption{Quantitative comparison of art-style concept erasure. Target concepts are evaluated by CLIP score (CS, lower is better), while non-target concepts are evaluated by FID (lower is better). Best results are shown in bold and second-best results are underlined.}
\label{tab:style_erasure}
\begin{tabular}{@{}l|ccccc@{}}
\toprule
\textbf{Concept}
& \textbf{Van Gogh}
& \textbf{Picasso}
& \textbf{Monet}
& \textbf{Andy Warhol}
& \textbf{Caravaggio} \\
\midrule
SD v1.4
& 29.21 & 29.06 & 29.02 & -- & -- \\
\midrule

\multicolumn{6}{@{}l}{\emph{Erase Van Gogh}} \\
\cmidrule(lr){1-6}
& CS$\downarrow$ & FID$\downarrow$ & FID$\downarrow$ & FID$\downarrow$ & FID$\downarrow$ \\
ConAbl~\cite{kumari2023ablating}
& 28.80 & 71.71 & 138.72 & 70.30 & 73.10 \\

MACE~\cite{lu2024mace}
& 27.74 & 65.77 & 69.79 & 83.37 & 75.41 \\
SPM~\cite{lyu2024one}
& \underline{24.78} & 62.25 & 32.27 & 58.30 & 61.50 \\
NP~\cite{schramowski2023safe}
& 24.90 & 141.56 & 124.52 & 127.85 & 136.32 \\
SLD~\cite{schramowski2023safe}
& 27.48 & 103.96 & 109.11 & 103.89 & 119.32 \\
AdaVD~\cite{wang2025precise}
& 24.87 & \underline{6.82} & \underline{2.66} & \textbf{8.36} & \underline{6.84} \\
Ours $(\gamma = 0.2)$
& \textbf{23.63} & \textbf{5.86} & \textbf{2.44} & \underline{8.59} & \textbf{4.80} \\
\midrule

\multicolumn{6}{@{}l}{\emph{Erase Picasso}} \\
\cmidrule(lr){1-6}
& FID$\downarrow$ & CS$\downarrow$ & FID$\downarrow$ & FID$\downarrow$ & FID$\downarrow$ \\
ConAbl~\cite{kumari2023ablating}
& 58.62 & 27.72 & 140.34 & 73.35 & 67.44 \\
MACE~\cite{lu2024mace}
& 60.46 & 27.11 & 49.92 & 76.10 & 72.85 \\
SPM~\cite{lyu2024one}
& 38.79 & \underline{26.69} & 7.76 & 52.00 & 51.40 \\
NP~\cite{schramowski2023safe}
& 111.35 & \textbf{26.14} & 91.11 & 116.24 & 121.82 \\
SLD~\cite{schramowski2023safe}
& 98.21 & 27.03 & 93.01 & 97.00 & 110.05 \\
AdaVD~\cite{wang2025precise}
& \underline{5.49} & 26.99 & \underline{2.33} & \underline{9.38} & \underline{7.05} \\
Ours $(\gamma = 0.2)$
& \textbf{5.18} & 26.99 & \textbf{1.69} & \textbf{7.53} & \textbf{4.36} \\
\midrule

\multicolumn{6}{@{}l}{\emph{Erase Monet}} \\
\cmidrule(lr){1-6}
& FID$\downarrow$ & FID$\downarrow$ & CS$\downarrow$ & FID$\downarrow$ & FID$\downarrow$ \\
ConAbl~\cite{kumari2023ablating}
& 141.52 & 132.10 & \underline{24.53} & 208.38 & 186.26 \\

MACE~\cite{lu2024mace}
& 76.90 & 69.35 & 26.89 & 88.35 & 81.72 \\
SPM~\cite{lyu2024one}
& 41.03 & 29.71 & 27.00 & 31.90 & 25.99 \\
NP~\cite{schramowski2023safe}
& 137.21 & 126.75 & \textbf{24.47} & 127.22 & 135.83 \\
SLD~\cite{schramowski2023safe}
& 94.48 & 92.88 & 25.73 & 100.90 & 114.87 \\
AdaVD~\cite{wang2025precise}
& \underline{6.94} & \underline{6.50} & 26.30 & \underline{8.46} & \underline{7.19} \\
Ours $(\gamma = 0.2)$
& \textbf{5.85} & \textbf{5.46} & 24.58 & \textbf{7.21} & \textbf{4.22} \\
\bottomrule
\end{tabular}%
\end{table}

For Monet erasure, CARE reduces target CS from 26.30 to 24.58, close to the strongest erasure baselines, while giving the best preservation across all retained styles. Van Gogh FID decreases from 6.94 to 5.85, Picasso from 6.50 to 5.46, Andy Warhol from 8.46 to 7.21, and Caravaggio from 7.19 to 4.22. Overall, the style results strongly support the central claim of the method: retained-subspace-aware erasure can reduce collateral damage while maintaining competitive target removal.

\subsection{Celebrity Concept Erasure}

\begin{table}[!t]
\centering
\caption{Quantitative comparison of celebrity concept erasure. For the erased concept, CLIP score (CS, lower is better) measures erasure efficacy. For non-target concepts, FID (lower is better) measures prior preservation. Greyed entries are not the primary metric for that block. Best results are shown in bold and second-best results are underlined.}
\label{tab:celebrity_erasure}
\resizebox{\textwidth}{!}{%
\begin{tabular}{@{}l|cc|cc|cc|cc|cc@{}}
\toprule
\textbf{Concept}
& \multicolumn{2}{c|}{\textbf{Bruce Lee}}
& \multicolumn{2}{c|}{\textbf{Marilyn Monroe}}
& \multicolumn{2}{c|}{\textbf{Melania Trump}}
& \multicolumn{2}{c|}{\textbf{Anne Hathaway}}
& \multicolumn{2}{c}{\textbf{Tom Cruise}} \\
& CS & FID & CS & FID & CS & FID & CS & FID & CS & FID \\
\midrule
SD v1.4
& 30.77 & -- & 27.70 & -- & 29.80 & -- & 31.96 & -- & 31.12 & -- \\
\midrule

\multicolumn{11}{@{}c}{\emph{Erase Bruce Lee}} \\
\cmidrule(lr){1-11}
& CS$\downarrow$ & \gray{FID}
& \gray{CS} & FID$\downarrow$
& \gray{CS} & FID$\downarrow$
& \gray{CS} & FID$\downarrow$
& \gray{CS} & FID$\downarrow$ \\
ConAbl~\cite{kumari2023ablating}
& 31.35 & \gray{87.57} & \gray{28.23} & 57.79 & \gray{29.77} & 40.95 & \gray{29.77} & 40.95 & \gray{30.97} & 53.53 \\
MACE~\cite{lu2024mace}
& 25.04 & \gray{131.29} & \gray{28.13} & 74.80 & \gray{30.07} & 68.83 & \gray{31.91} & 75.05 & \gray{28.13} & 71.20 \\
SPM~\cite{lyu2024one}
& 27.75 & \gray{123.67} & \gray{27.71} & 26.89 & \gray{29.81} & 7.83 & \gray{31.96} & 9.46 & \gray{31.13} & 28.54 \\
NP~\cite{schramowski2023safe}
& 24.70 & \gray{150.85} & \gray{26.84} & 102.67 & \gray{28.94} & 82.13 & \gray{30.34} & 89.60 & \gray{29.67} & 89.92 \\
SLD~\cite{schramowski2023safe}
& 28.22 & \gray{102.26} & \gray{26.29} & 87.15 & \gray{29.43} & 84.32 & \gray{30.97} & 85.37 & \gray{29.32} & 94.07 \\
AdaVD~\cite{wang2025precise}
& \underline{20.67} & \gray{138.70} & \gray{27.70} & \underline{6.68} & \gray{29.82} & \underline{5.08} & \gray{31.97} & \underline{6.39} & \gray{31.10} & \underline{13.11} \\
CARE (ours)
& \textbf{18.42} & \gray{135.65} & \gray{27.73} & \textbf{3.29} & \gray{29.65} & \textbf{4.37} & \gray{31.95} & \textbf{5.07} & \gray{31.13} & \textbf{5.89} \\
\midrule

\multicolumn{11}{@{}c}{\emph{Erase Marilyn Monroe}} \\
\cmidrule(lr){1-11}
& \gray{CS} & FID$\downarrow$
& CS$\downarrow$ & \gray{FID}
& \gray{CS} & FID$\downarrow$
& \gray{CS} & FID$\downarrow$
& \gray{CS} & FID$\downarrow$ \\
ConAbl~\cite{kumari2023ablating}
& \gray{30.88} & 66.97 & 28.75 & \gray{88.45} & \gray{29.69} & 51.52 & \gray{32.05} & 58.57 & \gray{31.10} & 54.13 \\
MACE~\cite{lu2024mace}
& \gray{31.30} & 76.23 & \underline{19.52} & \gray{148.34} & \gray{31.93} & 71.05 & \gray{30.16} & 74.90 & \gray{31.52} & 73.06 \\
SPM~\cite{lyu2024one}
& \gray{30.76} & 32.70 & 21.87 & \gray{145.81} & \gray{29.83} & 25.27 & \gray{31.96} & 22.86 & \gray{31.10} & 19.34 \\
NP~\cite{schramowski2023safe}
& \gray{29.50} & 113.12 & 25.86 & \gray{149.95} & \gray{29.29} & 87.27 & \gray{29.42} & 98.86 & \gray{30.02} & 86.70 \\
SLD~\cite{schramowski2023safe}
& \gray{29.59} & 87.83 & 26.70 & \gray{98.51} & \gray{28.81} & 107.42 & \gray{29.25} & 102.13 & \gray{30.35} & 81.12 \\
AdaVD~\cite{wang2025precise}
& \gray{30.73} & \underline{7.88} & 19.87 & \gray{116.94} & \gray{29.80} & \underline{4.46} & \gray{31.93} & \underline{5.43} & \gray{31.13} & \underline{9.33} \\
CARE (ours)
& \gray{30.64} & \textbf{5.92} & \textbf{17.73} & \gray{112.11} & \gray{29.66} & \textbf{4.36} & \gray{31.96} & \textbf{5.01} & \gray{31.15} & \textbf{4.53} \\
\midrule

\multicolumn{11}{@{}c}{\emph{Erase Melania Trump}} \\
\cmidrule(lr){1-11}
& \gray{CS} & FID$\downarrow$
& \gray{CS} & FID$\downarrow$
& CS$\downarrow$ & \gray{FID}
& \gray{CS} & FID$\downarrow$
& \gray{CS} & FID$\downarrow$ \\
ConAbl~\cite{kumari2023ablating}
& \gray{30.62} & 54.46 & \gray{28.14} & 59.10 & 29.89 & \gray{79.04} & \gray{31.94} & 58.65 & \gray{31.00} & 54.50 \\
MACE~\cite{lu2024mace}
& \gray{31.30} & 78.07 & \gray{27.84} & 71.34 & \textbf{20.71} & \gray{122.42} & \gray{31.94} & 73.49 & \gray{31.41} & 71.09 \\
SPM~\cite{lyu2024one}
& \gray{30.79} & 14.08 & \gray{27.63} & 30.40 & 23.12 & \gray{129.68} & \gray{31.86} & 28.85 & \gray{31.10} & 22.35 \\
NP~\cite{schramowski2023safe}
& \gray{29.38} & 115.35 & \gray{27.63} & 103.83 & 23.73 & \gray{131.73} & \gray{28.72} & 106.04 & \gray{30.27} & 106.00 \\
SLD~\cite{schramowski2023safe}
& \gray{29.55} & 90.69 & \gray{26.24} & 93.93 & 25.45 & \gray{103.52} & \gray{28.43} & 104.48 & \gray{30.47} & 88.31 \\
AdaVD~\cite{wang2025precise}
& \gray{30.75} & \underline{7.32} & \gray{27.69} & \underline{6.86} & 23.28 & \gray{96.66} & \gray{31.95} & \textbf{6.52} & \gray{31.08} & \underline{5.74} \\
CARE (ours)
& \gray{30.61} & \textbf{6.33} & \gray{27.70} & \textbf{4.06} & \underline{21.90} & \gray{96.22} & \gray{31.95} & \underline{6.85} & \gray{31.22} & \textbf{4.38} \\
\bottomrule
\end{tabular}%
}
\end{table}

Table~\ref{tab:celebrity_erasure} evaluates celebrity erasure. This setting is challenging because identities share a strong face prior, so removing one identity can easily affect others. CARE performs well in this regime.

When erasing Bruce Lee, CARE improves target erasure from AdaVD's 20.67 to 18.42. It also improves all retained identity FIDs reported in the table: Marilyn Monroe improves from 6.68 to 3.29, Melania Trump from 5.08 to 4.37, Anne Hathaway from 6.39 to 5.07, and Tom Cruise from 13.11 to 5.89. This is a strong example of the benefit of separating target-specific identity from shared face structure.

When erasing Marilyn Monroe, CARE again improves target erasure, reducing CS from 19.87 to 17.73. It also improves preservation on the retained identities: Bruce Lee FID decreases from 7.88 to 5.92, Melania Trump from 4.46 to 4.36, Anne Hathaway from 5.43 to 5.01, and Tom Cruise from 9.33 to 4.53. Compared with methods such as MACE or SPM, which can strongly suppress the target but substantially distort retained identities, CARE provides a more balanced operating point.

For Melania Trump, CARE reduces the target CS from 23.28 to 21.90, while improving most retained FIDs. Bruce Lee improves from 7.32 to 6.33, Marilyn Monroe from 6.86 to 4.06, and Tom Cruise from 5.74 to 4.38. Anne Hathaway slightly worsens from 6.52 to 6.85. This case illustrates the limits of a preservation-oriented operator: when the target is highly entangled with retained identities, CARE still improves the overall trade-off, but does not uniformly dominate every retained probe.

\subsection{Qualitative Analysis}

Figure~\ref{fig:qual_celebrity} and Figure~\ref{fig:qual_style_instance} provide qualitative examples of CARE. Each panel is a $3\times3$ montage, where columns correspond to SD v1.4, AdaVD, and CARE, and rows correspond to the erased target followed by two bystander concepts. These examples use a single template and seed, and are intended to illustrate the same erase--preserve behaviour measured quantitatively in Tables~\ref{tab:celebrity_erasure}, \ref{tab:style_erasure}, and \ref{tab:instance_erasure_detailed}. Across the examples, the target concept is visibly suppressed by the edited methods, while CARE better preserves the visual identity or style of the bystander rows. This is especially apparent for celebrity and style erasure, where the retained concepts share substantial structure with the erased target.

\begin{figure*}[!t]
    \centering
    \includegraphics[width=\textwidth]{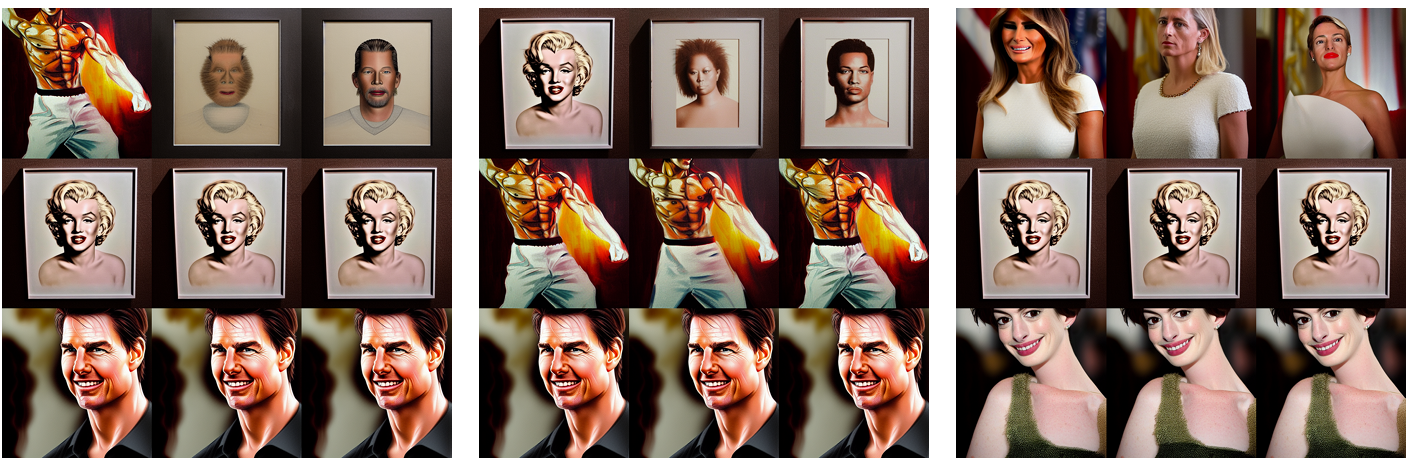}
    \caption{
    Qualitative celebrity erasure. Each $3\times3$ panel shows columns as SD v1.4 original, AdaVD, and CARE, and rows as erased target, bystander 1, and bystander 2. Left: erase Bruce Lee, with Marilyn Monroe and Tom Cruise as bystanders. Centre: erase Marilyn Monroe, with Bruce Lee and Tom Cruise as bystanders. Right: erase Melania Trump, with Marilyn Monroe and Anne Hathaway as bystanders. The edited columns suppress the target identity in the first row, while CARE better preserves the bystander identities, consistent with the FID improvements in Table~\ref{tab:celebrity_erasure}.
    }
    \label{fig:qual_celebrity}
\end{figure*}

\begin{figure*}[!t]
    \centering
    \includegraphics[width=\textwidth]{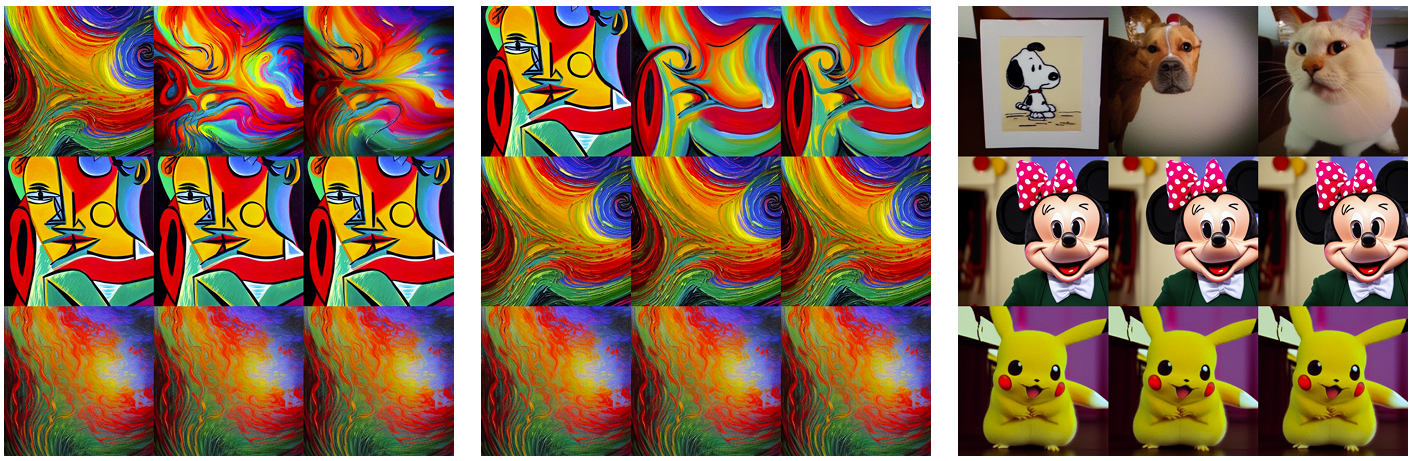}
    \caption{
    Qualitative style and instance erasure. Each $3\times3$ panel shows columns as SD v1.4 original, AdaVD, and CARE, and rows as erased target, bystander 1, and bystander 2. Left: erase Van Gogh style, with Picasso and Monet as bystanders. Centre: erase Picasso style, with Van Gogh and Monet as bystanders. Right: erase Snoopy, with Mickey Mouse and Pikachu as bystanders. CARE suppresses the target style or character while preserving the retained bystanders, matching the preservation trends reported in Tables~\ref{tab:style_erasure} and~\ref{tab:instance_erasure_detailed}.
    }
    \label{fig:qual_style_instance}
\end{figure*}

\subsection{Efficiency}
\begin{table}[!t]
\centering
\caption{
Time consumption of 10-concept erasure. Times are reported in seconds for 10 generated images on one NVIDIA A40, following AdaVD~\cite{wang2025precise}. Baseline rows are transcribed from AdaVD. CARE is closed-form and training-free; instead of absolute A40 timing, we report the measured overhead: a 1.2s offline retained-subspace build, $+0.7\%$ per-image generation overhead, and zero fine-tuning.
}
\label{tab:time_consumption}
\begin{tabular}{@{}l|c|c|c|c@{}}
\toprule
\textbf{Method}
& \textbf{Data Preparation}
& \textbf{Model Finetune}
& \textbf{Image Generation}
& \textbf{Time} \\
\midrule
ConAbl~\cite{kumari2023ablating}
& 9290 & 1120 & 0.9 & 10419 \\
ESD~\cite{gandikota2023erasing}
& -- & -- & -- & -- \\
SPM~\cite{lyu2024one}
& 0 & 72850 & 1.7 & 72867 \\
MACE~\cite{lu2024mace}
& 303 & 232 & 0.9 & 544 \\
NP~\cite{schramowski2023safe}
& -- & -- & -- & -- \\
SLD~\cite{schramowski2023safe}
& 0 & 0 & 1.4 & 14 \\
AdaVD~\cite{wang2025precise}
& 4 & \textbf{0} & 1.8 & 22 \\
CARE (ours)
& 1.2 & \textbf{0} & $+0.7\%$ & 22 \\
\bottomrule
\end{tabular}
\end{table}
Table~\ref{tab:time_consumption} reports the computational cost of CARE compared with prior methods. Training-based methods such as Concept Ablation, SPM, and MACE require model fine-tuning or additional preparation. In contrast, CARE is closed-form and training-free. It has zero model fine-tuning time, matching AdaVD in this respect. The only additional computation is an offline retained-subspace construction, measured at 1.2 seconds, and a small per-image generation overhead of $+0.7\%$. The total runtime is therefore effectively the same as AdaVD, while providing a retained-subspace-aware erasure direction.

\subsection{Ablation Studies}


\noindent\textbf{Effect of the shrinkage parameter.}
The left side of Table~\ref{tab:ablations_gamma_bank} evaluates the shrinkage parameter $\gamma$ on a separable concept, Bruce Lee, and an entangled concept, Melania Trump. Very small values of $\gamma$ over-protect the retained subspace and collapse erasure: for Bruce Lee, $\gamma=0.05$ and $\gamma=0.10$ give CS values near 30, close to the unedited model, even though FID is low. This confirms that preservation alone is not sufficient; the target must still be removed. For Bruce Lee, $\gamma=0.20$ gives the best erasure in the ablation, reducing CS to 17.19 compared with AdaVD's 18.95, while also improving the Tom Cruise probe FID from 11.72 to 6.53. For Melania Trump, which is more entangled with the retained face subspace, larger values of $\gamma$ improve erasure more strongly. This supports the interpretation of $\gamma$ as an erase--preserve dial: smaller values favour preservation, while larger values recover stronger target removal.

\noindent\textbf{Anchor-bank composition.}
The right side of Table~\ref{tab:ablations_gamma_bank} studies the retained anchor bank. The blind disjoint bank gives the strongest overall result among bank types, with CS 17.19 and FID 6.53. Random anchors weaken both erasure and preservation, unrelated anchors especially hurt FID, and related anchors perform worst, increasing CS to 22.88 and FID to 9.98. This is consistent with the method's motivation. If the anchor bank is too close to the target identity, the protected subspace absorbs the target-specific direction and erasure becomes weaker. A useful anchor bank should capture shared non-target structure while remaining disjoint from the target identity.

\begin{table}[!t]
\centering
\caption{
Ablation studies. Left: effect of the shrinkage parameter $\gamma$ on erasure CS$\downarrow$ and probe FID$\downarrow$ for a separable concept, Bruce Lee with Tom Cruise as probe, and an entangled concept, Melania Trump with Anne Hathaway as probe. Right: effect of anchor-bank composition and bank size $M$ for Bruce Lee erasure. All ablations use reduced-scope runs with 15 templates, NS=10, seed 0, and the same evaluation setting. Bold indicates the best result in each column. The star marks the deployed setting.
}
\label{tab:ablations_gamma_bank}

\begin{minipage}[t]{0.56\textwidth}
\centering
\resizebox{\linewidth}{!}{%
\begin{tabular}{@{}l|cc|cc@{}}
\toprule
\multirow{2}{*}{$\gamma$}
& \multicolumn{2}{c|}{\textbf{Bruce Lee}}
& \multicolumn{2}{c}{\textbf{Melania Trump}} \\
& CS$\downarrow$ & FID$\downarrow$
& CS$\downarrow$ & FID$\downarrow$ \\
\midrule
AdaVD $(\gamma \rightarrow \infty)$
& 18.95 & 11.72 & 22.33 & 6.93 \\
\midrule
0.05
& 29.99 & \textbf{4.09} & 29.46 & \textbf{4.21} \\
0.10
& 30.01 & 4.42 & 29.39 & 4.35 \\
0.20$^{\star}$
& \textbf{17.19} & 6.53 & 21.59 & 5.09 \\
0.50
& 17.27 & 7.56 & 21.17 & 5.64 \\
1.00
& 17.65 & 10.06 & \textbf{20.99} & 5.33 \\
\bottomrule
\end{tabular}%
}

\vspace{0.35em}
\footnotesize{(a) Shrinkage sensitivity.}
\end{minipage}
\hfill
\begin{minipage}[t]{0.41\textwidth}
\centering
\resizebox{\linewidth}{!}{%
\begin{tabular}{@{}l|cc||c|cc@{}}
\toprule
\textbf{Bank type}
& CS$\downarrow$ & FID$\downarrow$
& $M$ & CS$\downarrow$ & FID$\downarrow$ \\
\midrule
blind (disjoint) & \textbf{17.19} & \textbf{6.53}
& 2  & 17.96 & 6.13 \\
random                  & 19.02 & 6.85
& 4  & \textbf{16.44} & \textbf{5.72} \\
unrelated               & 19.04 & 8.77
& 6 & 17.19 & 6.53 \\
related                 & 22.88 & 9.98
& 8  & 17.25 & 6.25 \\
                       &       &
& 10 & 17.34 & 5.83 \\
\bottomrule
\end{tabular}%
}

\vspace{0.35em}
\footnotesize{(b) Anchor-bank composition and size.}
\end{minipage}

\end{table}

\noindent\textbf{Anchor-bank size.}
The bank-size ablation shows that CARE is not highly sensitive to the exact number of anchors. With $M=2$, CS is 17.96 and FID is 6.13. For $M \in \{4,6,8,10\}$, CS remains between 16.44 and 17.34, and FID remains between 5.72 and 6.53. We use $M=6$ as a conservative default because it provides stable behaviour with low overhead.

\section{Conclusion}

We presented CARE, a training-free concept-erasure method for text-to-image diffusion models that removes unwanted concepts while reducing collateral damage to non-target generations. CARE is motivated by the observation that the raw target direction used in value-space erasure can mix target-specific identity with visual structure shared by concepts that should be preserved. Instead of erasing along this raw direction, CARE computes a covariance-aware retained-subspace direction from a small bank of kept concept anchors, yielding a closed-form operator with no model fine-tuning and negligible offline overhead. Across instance, style, and celebrity erasure, CARE improves prior preservation while maintaining matched or competitive erasure, with the strongest gains on concepts whose target-specific component is separable from the retained subspace. More entangled cases expose the expected erase--preserve trade-off, which CARE controls through the shrinkage parameter $\gamma$. Overall, CARE reframes concept erasure as choosing the right direction to forget: suppressing specific unwanted concepts without broadly degrading the generative prior.

\FloatBarrier
\bibliographystyle{splncs04}
\bibliography{main}

\clearpage
\appendix

\section{Additional Method Details}
\label{sec:supp_method}

This section provides additional details for CARE that were omitted from the main paper for space: the low-rank Woodbury computation, the hard projection form and its kept-subspace invariance property, and the multi-concept extension.

\subsection{Efficient Woodbury Computation}

The CARE direction is defined in the main paper as
\begin{equation}
    d_j
    =
    \Sigma_{R,j}^{-1}t_j,
    \qquad
    \Sigma_{R,j}
    =
    \frac{1}{M}B_j^\top B_j + \gamma I_D .
    \label{eq:supp_care_direction}
\end{equation}
Directly inverting $\Sigma_{R,j} \in \mathbb{R}^{D \times D}$ is unnecessary and inefficient. Since the number of retained anchors $M$ is small compared with the value dimension $D$, we compute $d_j$ using the Woodbury identity:
\begin{equation}
    d_j =
    \frac{1}{\gamma}
    \left[
    t_j
    -
    \frac{1}{\gamma}
    B_j^{\top}
    \left(
        M I_M + \frac{1}{\gamma} B_j B_j^{\top}
    \right)^{-1}
    B_j t_j
    \right].
    \label{eq:supp_woodbury_direction}
\end{equation}
Thus the only matrix inverse is over an $M \times M$ matrix. The direction $d_j$ is computed once and cached for each edited cross-attention layer and token position. During sampling, CARE therefore has the same form of inference-time update as raw value-space erasure, with $d_j$ replacing $t_j$.

\subsection{Projection Form and Kept-Subspace Invariance}

CARE also admits a hard projection form that clarifies the preservation mechanism. Let $R_j \in \mathbb{R}^{D \times r}$ be an orthonormal basis for the retained subspace obtained from the anchor bank $B_j$. The retained-subspace-orthogonal component of the target is
\begin{equation}
    t^{\perp}_j =
    \left(I_D - R_j R_j^{\top}\right)t_j .
    \label{eq:supp_orthogonal_target}
\end{equation}
Using $t^{\perp}_j$ as the erasure direction yields
\begin{equation}
    v^{\mathrm{proj}}_j
    =
    v_j
    -
    \delta\!\left(\cos(t_j, v_j)\right)
    \frac{\langle t^{\perp}_j, v_j\rangle}
         {\langle t^{\perp}_j, t^{\perp}_j\rangle}
    t^{\perp}_j .
    \label{eq:supp_projection_update}
\end{equation}
This projection form leaves any value vector lying in the retained subspace unchanged. If $w \in \mathrm{span}(R_j)$, then
\begin{equation}
    \langle t^{\perp}_j, w\rangle
    =
    \left\langle
    \left(I_D - R_jR_j^\top\right)t_j,
    w
    \right\rangle
    =
    0 ,
\end{equation}
because $t^{\perp}_j$ lies in the orthogonal complement of $\mathrm{span}(R_j)$. Therefore, the update coefficient in Eq.~\eqref{eq:supp_projection_update} vanishes for any retained-subspace vector $w$. This gives an exact kept-subspace invariance property for the projection form.

The covariance-aware direction used by CARE in the main paper can be viewed as a soft version of this projection. Rather than deleting retained directions outright, it down-weights them according to the retained-anchor covariance. This avoids the brittleness of a hard subspace cutoff while still reducing interference with retained concepts.

\subsection{Multi-Concept Erasure}

CARE extends naturally to multiple target concepts. Suppose we wish to erase a set of target concepts
\begin{equation}
    \mathcal{C} = \{c_1,\ldots,c_K\}.
\end{equation}
For each target concept $c_k$, we record a target value $t_j^{(k)}$ and compute a covariance-aware direction
\begin{equation}
    d_j^{(k)}
    =
    \left(
    \frac{1}{M}B_j^\top B_j + \gamma I_D
    \right)^{-1}
    t_j^{(k)} ,
\end{equation}
using the same retained anchor bank. The multi-concept CARE update is then
\begin{equation}
    v^{\mathrm{CARE}}_j
    =
    v_j
    -
    \sum_{k=1}^{K}
    \delta\!\left(\cos(t_j^{(k)}, v_j)\right)
    \frac{\langle d_j^{(k)}, v_j\rangle}
         {\langle d_j^{(k)}, d_j^{(k)}\rangle}
    d_j^{(k)} .
    \label{eq:supp_multi_concept}
\end{equation}
In practice, the directions can be orthonormalized per token before summation, following standard multi-concept value-space erasure, to reduce redundancy between closely related target concepts.

\subsection{Computational Cost}

CARE does not update the U-Net, the text encoder, or any cross-attention projection weights. The additional computation relative to raw value-space erasure is the offline construction of $d_j$ for each edited layer and token position. Because Eq.~\eqref{eq:supp_woodbury_direction} only requires an $M \times M$ inverse, this cost is small when the retained bank contains few anchors. Once $d_j$ is cached, inference consists of the same gated rank-one subtraction used in standard value-space erasure.

\section{Additional Theoretical Properties}
\label{sec:supp_theory}

We provide a few simple properties of CARE that clarify the behaviour of the proposed operator. These results are not required to implement the method, but they help explain why the covariance-aware direction gives a controllable erase--preserve trade-off.

\subsection{Minimum-Disturbance Interpretation}

The CARE update can be interpreted as the smallest change to a value vector that reduces its component along the erasure direction.

\begin{proposition}[Minimum-disturbance value update]
\label{prop:min_disturbance}
Let $d_j \neq 0$ be a fixed erasure direction and let $\delta \in [0,1]$ be a scalar gate. The CARE update
\begin{equation}
    v'_j
    =
    v_j
    -
    \delta
    \frac{\langle d_j, v_j\rangle}{\langle d_j, d_j\rangle}
    d_j
\end{equation}
is the solution of
\begin{equation}
    \min_{z \in \mathbb{R}^{D}} \|z-v_j\|_2^2
    \quad
    \text{s.t.}
    \quad
    \langle d_j,z\rangle
    =
    (1-\delta)\langle d_j,v_j\rangle .
\end{equation}
In particular, when $\delta=1$, the update is the closest vector to $v_j$ whose component along $d_j$ is completely removed.
\end{proposition}

\begin{proof}
The constraint requires the edited value $z$ to have a prescribed inner product with $d_j$. The closest point to $v_j$ satisfying this affine constraint must lie on the line $v_j-\alpha d_j$. Substituting $z=v_j-\alpha d_j$ into the constraint gives
\begin{equation}
    \langle d_j,v_j\rangle
    -
    \alpha \langle d_j,d_j\rangle
    =
    (1-\delta)\langle d_j,v_j\rangle .
\end{equation}
Solving for $\alpha$ yields
\begin{equation}
    \alpha
    =
    \delta
    \frac{\langle d_j,v_j\rangle}{\langle d_j,d_j\rangle},
\end{equation}
which gives the CARE update.
\end{proof}

\subsection{Relation to Raw Value-Space Erasure}

CARE recovers standard value-space erasure as a limiting case when the retained covariance is suppressed.

\begin{proposition}[Raw erasure as a limiting case]
\label{prop:raw_limit}
Let
\begin{equation}
    d_j(\gamma)
    =
    \left(
    \frac{1}{M}B_j^\top B_j + \gamma I_D
    \right)^{-1}t_j .
\end{equation}
As $\gamma \rightarrow \infty$, the CARE update converges to the raw value-space erasure update using direction $t_j$.
\end{proposition}

\begin{proof}
For large $\gamma$,
\begin{equation}
    \left(
    \frac{1}{M}B_j^\top B_j + \gamma I_D
    \right)^{-1}
    =
    \frac{1}{\gamma}I_D + O(\gamma^{-2}).
\end{equation}
Therefore,
\begin{equation}
    d_j(\gamma)
    =
    \frac{1}{\gamma}t_j + O(\gamma^{-2}).
\end{equation}
The CARE update is invariant to non-zero rescaling of the erasure direction, since
\begin{equation}
    \frac{\langle \alpha d_j,v_j\rangle}
         {\langle \alpha d_j,\alpha d_j\rangle}
    \alpha d_j
    =
    \frac{\langle d_j,v_j\rangle}
         {\langle d_j,d_j\rangle}
    d_j
\end{equation}
for any $\alpha \neq 0$. Hence, as $\gamma \rightarrow \infty$, CARE becomes equivalent to erasing along $t_j$.
\end{proof}

\subsection{Hard-Projection Limit and Retained-Subspace Invariance}

CARE also approaches a hard projection away from the retained-anchor subspace when the shrinkage parameter becomes small.

\begin{proposition}[Hard-projection limit]
\label{prop:hard_projection_limit}
Let $R_j \in \mathbb{R}^{D \times r}$ be an orthonormal basis for the subspace spanned by the retained anchors, and define
\begin{equation}
    t_j^\perp
    =
    (I_D - R_jR_j^\top)t_j .
\end{equation}
If $t_j^\perp \neq 0$, then as $\gamma \rightarrow 0$, the CARE direction becomes equivalent, up to scale, to $t_j^\perp$.
\end{proposition}

\begin{proof}
Let the eigendecomposition of $\frac{1}{M}B_j^\top B_j$ be
\begin{equation}
    \frac{1}{M}B_j^\top B_j
    =
    U_r \Lambda U_r^\top ,
\end{equation}
where $U_r$ spans the retained-anchor subspace and $\Lambda$ contains the non-zero eigenvalues. Decompose the target as
\begin{equation}
    t_j = U_rU_r^\top t_j + t_j^\perp .
\end{equation}
Then
\begin{equation}
    d_j(\gamma)
    =
    U_r(\Lambda+\gamma I)^{-1}U_r^\top t_j
    +
    \frac{1}{\gamma}t_j^\perp .
\end{equation}
As $\gamma \rightarrow 0$, the orthogonal component $\frac{1}{\gamma}t_j^\perp$ dominates whenever $t_j^\perp \neq 0$. Since the value update is invariant to rescaling of $d_j$, the effective erasure direction converges to $t_j^\perp$.
\end{proof}

\begin{corollary}[Retained-subspace invariance]
\label{cor:retained_invariance}
In the hard-projection form, any value vector $w \in \mathrm{span}(R_j)$ is left unchanged by the erasure update.
\end{corollary}

\begin{proof}
The hard-projection direction is $t_j^\perp=(I_D-R_jR_j^\top)t_j$, which lies in the orthogonal complement of $\mathrm{span}(R_j)$. Therefore, for any $w \in \mathrm{span}(R_j)$,
\begin{equation}
    \langle t_j^\perp,w\rangle = 0 .
\end{equation}
The projection coefficient in the erasure update is therefore zero, so the update leaves $w$ unchanged.
\end{proof}

\begin{table}[!t]
\centering
\caption{
Detailed quantitative comparison of single- and multi-instance concept erasure. For erased concepts, CLIP score (CS, lower is better) measures erasure efficacy. For non-target concepts, FID (lower is better) measures prior preservation. Greyed entries are not the primary metric for that block. Best results are shown in bold and second-best results are underlined.
}
\label{tab:instance_erasure_detailed}
\resizebox{\textwidth}{!}{%
\begin{tabular}{@{}l|cc|cc|cc|cc|cc|cc@{}}
\toprule
\textbf{Method}
& \multicolumn{2}{c|}{\textbf{Snoopy}}
& \multicolumn{2}{c|}{\textbf{Mickey}}
& \multicolumn{2}{c|}{\textbf{Spongebob}}
& \multicolumn{2}{c|}{\textbf{Pikachu}}
& \multicolumn{2}{c|}{\textbf{Dog}}
& \multicolumn{2}{c}{\textbf{Legislator}} \\
& CS & FID & CS & FID & CS & FID & CS & FID & CS & FID & CS & FID \\
\midrule
SD v1.4
& 28.49 & -- & 26.50 & -- & 27.30 & -- & 27.41 & -- & 24.27 & -- & 23.73 & -- \\
\midrule

\multicolumn{13}{@{}c}{\emph{Erase Snoopy}} \\
\cmidrule(lr){1-13}
& CS$\downarrow$ & \gray{FID}
& \gray{CS} & FID$\downarrow$
& \gray{CS} & FID$\downarrow$
& \gray{CS} & FID$\downarrow$
& \gray{CS} & FID$\downarrow$
& \gray{CS} & FID$\downarrow$ \\
ConAbl~\cite{kumari2023ablating}
& 25.38 & \gray{103.80} & \gray{26.80} & 38.44 & \gray{27.02} & 41.59 & \gray{27.57} & 29.68 & \gray{24.32} & 27.76 & \gray{23.48} & 27.36 \\
MACE~\cite{lu2024mace}
& 20.78 & \gray{169.22} & \gray{22.95} & 118.01 & \gray{23.33} & 111.90 & \gray{25.77} & 81.99 & \gray{23.96} & 43.27 & \gray{22.25} & 65.97 \\
SPM~\cite{lyu2024one}
& 23.89 & \gray{122.63} & \gray{26.60} & 33.06 & \gray{27.12} & 34.70 & \gray{27.51} & 23.89 & \gray{24.24} & 19.61 & \gray{23.70} & 18.26 \\
NP~\cite{schramowski2023safe}
& 23.66 & \gray{125.98} & \gray{26.14} & 59.58 & \gray{26.66} & 78.74 & \gray{27.36} & 52.37 & \gray{23.89} & 67.51 & \gray{22.16} & 55.22 \\
SLD~\cite{schramowski2023safe}
& 27.84 & \gray{64.78} & \gray{26.46} & 48.12 & \gray{27.52} & 55.36 & \gray{27.33} & 38.74 & \gray{24.03} & 41.95 & \gray{22.80} & 49.08 \\
AdaVD~\cite{wang2025precise}
& \underline{20.28} & \gray{120.46} & \gray{26.53} & \underline{5.72} & \gray{27.25} & \underline{8.56} & \gray{27.40} & \underline{5.79} & \gray{24.27} & \underline{2.32} & \gray{23.77} & \underline{6.07} \\
CARE (ours)
& \textbf{19.42} & \gray{108.95} & \gray{23.92} & \textbf{4.43} & \gray{22.74} & \textbf{7.42} & \gray{25.86} & \textbf{4.74} & \gray{24.24} & \textbf{1.95} & \gray{22.14} & \textbf{5.32} \\
\midrule

\multicolumn{13}{@{}c}{\emph{Erase Snoopy and Mickey}} \\
\cmidrule(lr){1-13}
& CS$\downarrow$ & \gray{FID}
& CS$\downarrow$ & \gray{FID}
& \gray{CS} & FID$\downarrow$
& \gray{CS} & FID$\downarrow$
& \gray{CS} & FID$\downarrow$
& \gray{CS} & FID$\downarrow$ \\
ConAbl~\cite{kumari2023ablating}
& 24.26 & \gray{119.96} & 24.08 & \gray{96.94} & \gray{27.02} & 46.32 & \gray{27.75} & 39.63 & \gray{23.98} & 30.57 & \gray{23.33} & 27.49 \\
MACE~\cite{lu2024mace}
& 20.74 & \gray{171.16} & \underline{20.71} & \gray{140.50} & \gray{25.87} & 51.49 & \gray{25.87} & 110.67 & \gray{23.82} & 52.07 & \gray{21.70} & 77.13 \\
SPM~\cite{lyu2024one}
& 23.16 & \gray{128.08} & 22.81 & \gray{115.02} & \gray{26.92} & 41.58 & \gray{27.45} & 31.77 & \gray{24.13} & 21.96 & \gray{23.60} & 23.69 \\
NP~\cite{schramowski2023safe}
& 23.59 & \gray{124.10} & 24.85 & \gray{83.68} & \gray{26.69} & 81.41 & \gray{27.27} & 50.10 & \gray{23.62} & 65.93 & \gray{23.84} & 58.88 \\
SLD~\cite{schramowski2023safe}
& 27.76 & \gray{59.97} & 26.74 & \gray{50.16} & \gray{27.53} & 54.59 & \gray{27.29} & 39.24 & \gray{23.97} & 41.62 & \gray{22.66} & 50.13 \\
AdaVD~\cite{wang2025precise}
& \underline{20.29} & \gray{121.12} & \textbf{19.93} & \gray{108.22} & \gray{27.27} & \underline{9.34} & \gray{27.42} & \underline{5.84} & \gray{24.26} & \underline{2.41} & \gray{23.73} & \underline{6.43} \\
CARE (ours)
& \textbf{19.45} & \gray{108.77} & 22.55 & \gray{83.35} & \gray{27.24} & \textbf{7.47} & \gray{27.87} & \textbf{4.98} & \gray{24.12} & \textbf{2.18} & \gray{23.24} & \textbf{5.34} \\
\midrule

\multicolumn{13}{@{}c}{\emph{Erase Snoopy and Mickey and Spongebob}} \\
\cmidrule(lr){1-13}
& CS$\downarrow$ & \gray{FID}
& CS$\downarrow$ & \gray{FID}
& CS$\downarrow$ & \gray{FID}
& \gray{CS} & FID$\downarrow$
& \gray{CS} & FID$\downarrow$
& \gray{CS} & FID$\downarrow$ \\
ConAbl~\cite{kumari2023ablating}
& 23.94 & \gray{126.70} & 23.64 & \gray{105.07} & 25.04 & \gray{108.67} & \gray{27.76} & 51.20 & \gray{23.83} & 31.59 & \gray{23.17} & 30.03 \\
MACE~\cite{lu2024mace}
& 20.48 & \gray{172.80} & \underline{20.50} & \gray{143.66} & 21.59 & \gray{120.87} & \gray{24.38} & 99.68 & \gray{23.70} & 47.46 & \gray{21.74} & 70.38 \\
SPM~\cite{lyu2024one}
& 22.81 & \gray{133.06} & 22.35 & \gray{121.85} & \underline{20.82} & \gray{152.72} & \gray{27.45} & 39.83 & \gray{24.10} & 22.68 & \gray{23.52} & 25.31 \\
NP~\cite{schramowski2023safe}
& 24.29 & \gray{129.75} & 24.76 & \gray{89.74} & 25.31 & \gray{106.30} & \gray{27.28} & 64.75 & \gray{23.55} & 65.10 & \gray{21.63} & 59.33 \\
SLD~\cite{schramowski2023safe}
& 27.84 & \gray{58.16} & 26.71 & \gray{49.70} & 27.60 & \gray{54.61} & \gray{27.35} & 39.41 & \gray{23.90} & 42.32 & \gray{22.46} & 49.88 \\
AdaVD~\cite{wang2025precise}
& \textbf{19.39} & \gray{124.49} & \textbf{19.73} & \gray{112.97} & 20.34 & \gray{118.47} & \gray{27.42} & \underline{6.85} & \gray{24.27} & \underline{2.79} & \gray{23.76} & \underline{7.26} \\
CARE (ours)
& \underline{19.45} & \gray{108.67} & 22.56 & \gray{84.47} & \textbf{18.40} & \gray{81.82} & \gray{27.14} & \textbf{6.83} & \gray{24.23} & \textbf{2.27} & \gray{21.54} & \textbf{5.55} \\
\bottomrule
\end{tabular}%
}
\end{table}

\section{Additional Instance Erasure Results}
\label{sec:supp_instance_detailed}

Table~\ref{tab:instance_erasure_detailed} provides the full instance-erasure results corresponding to the compact comparison in the main paper. Unlike the main table, which reports only the primary metric for each concept in each erasure setting, this table includes both CLIP score (CS) and FID for every concept. This makes it possible to inspect both sides of the erase--preserve trade-off: for erased concepts, lower CS indicates stronger removal, while for non-target concepts, lower FID indicates better preservation of the original model's generation behaviour.

Greyed entries are included for completeness but are not the primary metric for that block. For example, when erasing Snoopy, the Snoopy CS column measures erasure efficacy, while the FID columns for Mickey, Spongebob, Pikachu, Dog, and Legislator measure preservation. Conversely, the Snoopy FID and the non-target CS columns are secondary diagnostic values and are therefore greyed out. This layout follows the same interpretation used in the main paper and helps avoid comparing metrics that are not meaningful for the corresponding target/non-target role.

Overall, the detailed results support the trends reported in the main paper. CARE improves single-concept Snoopy erasure while reducing FID for all retained concepts. In the multi-concept settings, CARE consistently improves retained-concept FID over AdaVD, although Mickey erasure remains weaker when Mickey is erased jointly with other targets. These results illustrate the intended behaviour of CARE: it shifts the operating point toward stronger prior preservation while maintaining competitive target removal.

\section{Discussion}
\label{sec:supp_discussion}

The results highlight the intended behaviour of CARE: it improves the specificity of value-space concept erasure by changing the direction of removal rather than simply increasing the strength of suppression. Across instance, style, and celebrity erasure, CARE consistently reduces changes to retained concepts while maintaining competitive target removal. This is especially clear in settings where the target has a distinguishable component relative to the retained anchors, such as Snoopy, art styles, Bruce Lee, and Marilyn Monroe. In these cases, the retained-anchor covariance helps separate target-specific information from shared visual structure, allowing CARE to suppress the target without unnecessarily disturbing nearby concepts.

The multi-concept and identity-erasure settings further illustrate the role of the erase--preserve trade-off. When target concepts are visually close to retained concepts, or when multiple related targets are erased jointly, a purely aggressive erasure direction can improve target suppression at the cost of broader prior degradation. CARE instead favours a more preservation-aware operating point. This is why some settings show a small reduction in erasure strength for one target while improving retained-concept FID across the remaining probes. Rather than treating this as a failure mode, we view it as the central advantage of the method: CARE exposes a controllable mechanism for deciding how much of the target direction should be removed once shared retained structure is accounted for.

This behaviour is consistent with the design of the operator. Standard value-space erasure removes the component of each value vector along the raw target direction, implicitly assuming that this direction contains only the concept to be erased. CARE relaxes this assumption by estimating which components of the target direction overlap with retained concepts and down-weighting those components before applying the value update. The result is a closed-form, training-free erasure method that preserves the practical efficiency of AdaVD-style value editing while adding an explicit mechanism for prior preservation.

Overall, CARE should be understood as a more selective value-space erasure operator. It does not require model fine-tuning, edited checkpoints, or additional trainable modules, and it introduces only a small offline computation to construct the retained-subspace-aware direction. Its empirical gains are therefore most meaningful in deployment settings where concept erasure must be fast, lightweight, and local, but where preserving semantically related non-target concepts is as important as removing the target itself.

\end{document}